%% file: main.tex
\documentclass{article}
\usepackage{authblk}
\usepackage[margin=1in]{geometry}
\pdfoutput=1

\usepackage{graphicx} % more modern
\usepackage{subcaption}

\usepackage[numbers]{natbib}

\usepackage{algorithm}
\usepackage{algorithmic}

\usepackage[T1]{fontenc}    % use 8-bit T1 fonts
\usepackage{url}            % simple URL typesetting
\usepackage{booktabs}       % professional-quality tables
\usepackage{amsfonts}       % blackboard math symbols
\usepackage{nicefrac}       % compact symbols for 1/2, etc.
\usepackage{microtype}      % microtypography
\usepackage{array}
\usepackage{pbox}

\usepackage{times}
\usepackage{mathtools}
\usepackage{color}
\usepackage{multirow}
\usepackage{amsthm}
\usepackage{mathrsfs}
\usepackage{graphicx}
\usepackage{caption}%, subcaption}
\usepackage{xspace}
\usepackage{float}

\usepackage{wrapfig}
\DeclareMathOperator*{\invK}{\text{inv}\bK}
\DeclareMathOperator*{\ztest}{\bz_\text{test}}

\input{notation.tex}

\input{theorem_notation.tex}

\author[1]{Rajiv Khanna}
\author[2]{Been Kim}
\author[3]{Joydeep Ghosh}
\author[4]{Oluwasanmi Koyejo}
\affil[1]{Department of Statistics \authorcr University of California at Berkeley \authorcr \texttt{rajivak@berkeley.edu}}
\affil[2]{Google Brain\authorcr \texttt{beenkim.mit@gmail.com}}
\affil[3]{Department of Electrical and Computer Engineering \authorcr University of Texas at Austin \authorcr \texttt{jghosh@utexas.edu}}
\affil[4]{Department of Computer Science \authorcr University of Illinois at Urbana-Champaign \authorcr \texttt{sanmi@illinois.edu}}

\begin{document}

\title{Interpreting Black Box Predictions using Fisher Kernels}
\maketitle
%\author{ Rajiv Khanna \\UC Berkeley \\ \And Author 2 \And  Author 3 }

\newcommand{\fix}{\marginpar{FIX}}
\newcommand{\new}{\marginpar{NEW}}
\input{abstract}
\input{intro}
\input{background}
\input{fisher}

\input{submod}

\input{influence}
\input{experiments}

\input{conclusion}

%\clearpage
%\setlength{\bibsep}{1.5pt plus .5ex}
\bibliographystyle{plainnat}

\bibliography{bibliography}
\clearpage
\appendix
\input{appendix.tex}

\end{document}

%% file: notation.tex
\mathchardef\mhyphen="2D

\DeclareMathOperator*{\argmin}{arg\,min}

\newcommand{\vertiii}[1]{{\left\vert\kern-0.25ex\left\vert\kern-0.25ex\left\vert #1
    \right\vert\kern-0.25ex\right\vert\kern-0.25ex\right\vert}}

\def\bb{{\mathbf{b}}}

\def\bff{{\mathbf{f}}}

\def\bk{{\mathbf{k}}}

\def\bt{{\mathbf{t}}}

\def\bw{{\mathbf{w}}}
\def\bx{{\mathbf{x}}}
\def\by{{\mathbf{y}}}
\def\bz{{\mathbf{z}}}
\def\bA{{\mathbf{A}}}

\def\bH{{\mathbf{H}}}

\def\bK{{\mathbf{K}}}

\def\bT{{\mathbf{T}}}

\def\bX{{\mathbf{X}}}
\def\bY{{\mathbf{Y}}}

\def\0{{\mathbf{0}}}

\def\bbE{{\mathbb{E}}}

\def\bbR{{\mathbb{R}}}

\def\cF{\mathcal{F}}

\def\cH{\mathcal{H}}
\def\cI{\mathcal{I}}

\def\cL{\mathcal{L}}

\def\cX{\mathcal{X}}

\def\sfA{\mathsf{A}}
\def\sfB{\mathsf{B}}

\def\sfL{\mathsf{L}}

\def\sfS{\mathsf{S}}

\def\frp{{\mathfrak{p}}}

%% file: theorem_notation.tex
\newtheorem*{rep@theorem}{\rep@title}
\newcommand{\newreptheorem}[2]{%
\newenvironment{rep#1}[1]{%
 \def\rep@title{#2 \ref{##1}}%
 \begin{rep@theorem}}%
 {\end{rep@theorem}}}
\newreptheorem{lemma}{Lemma'}
\newreptheorem{definition}{Definition'}
\newreptheorem{proposition}{Proposition'}
\newreptheorem{theorem}{Theorem'}

\newtheorem{theorem}{Theorem}
\newtheorem{definition}[theorem]{Definition}
\newtheorem{lemma}[theorem]{Lemma}
\newtheorem{proposition}[theorem]{Proposition}
\renewcommand{\text}[1]{{\textnormal{#1}}}

%% file: abstract.tex
\begin{abstract}
Research in both machine learning and psychology suggests that salient examples can help humans to interpret learning models. To this end, we take a novel look at black box interpretation of test predictions in terms of training examples. Our goal is to ask ``which training examples are most responsible for a given set of predictions''? To answer this question, we make use of Fisher kernels as the defining feature embedding of each data point, combined with Sequential Bayesian Quadrature (SBQ) for efficient selection of examples. In contrast to prior work, our method is able to seamlessly handle any sized subset of test predictions in a principled way. We theoretically analyze our approach, providing novel convergence bounds for SBQ over discrete candidate atoms. Our approach recovers the application of influence functions for interpretability as a special case yielding novel insights from this connection. We also present applications of the proposed approach to three use cases: cleaning training data, fixing mislabeled examples and data summarization.
\end{abstract}

%% file: intro.tex
\section{Introduction}\label{sec:introduction}

%Research in the machine learning and psychology literature suggest that salient examples are effective for enabling the interpretability of black-box machine learning models. Interpretability can be very important in many applications to assist a human who is involved in a decisive capacity. An interpretative suggestion can provide insights to be coupled with human intelligence that would have otherwise been overlooked in a more complicated non-interpretative model. 

It has long been established that using examples to enable interpretability is one of the most effective approaches for human learning and understanding~\cite{newell1972human, cohen1996metarecognition, kim2014bayesian}. The ability to interpret using examples from the data can lead to more informed decision based systems and a better understanding of the inner workings of the model~\cite{KohL17,Kim2016}. In this work, we are interested in finding data points or prototypes that are ``most responsible'' for the underlying model making specific predictions of interest. To this end, we develop a novel method that is model agnostic and only requires an access to the function and gradient oracles. 

In a more formal sense, we aim to approximate the empiricial test data distribution using samples from the training data. Our approach is to first embed all the points in the space induced by the Fisher kernels~\cite{Jaakkola1999}. This provides a principled way to quantify closeness of two points with respect to the similarity induced by the trained model.  If two points in this space are close, then intuitively the model treats them similarly. We formally show that influence function based approach to interpretability~\citep{KohL17} is essentially doing the same thing.

Thus, our goal is to find a subset of the training data such that, when also embedded in a model-induced space, is \emph{close} to the test set in the distribution sense. We build this subset from the training data sequentially using a greedy method called Sequential Bayesian Quadrature (SBQ)~\cite{Ohagan1991}. SBQ is an importance-sampling based algorithm to estimate the expected value of a function under a distribution using discrete sample points drawn from it. To the best of our knowledge SBQ has not been used in conjunction with Fisher kernels for interpretability. Moreover, we leverage recent research in discrete optimization to provide novel convergence rates for the algorithm over discrete atomic sets. Our analysis also yields novel and more scalable algorithm variants of SBQ with corresponding constant factor guarantees. 

%Our framework is flexible and has 
%The flexibility of our framework enables several important real world applications. For example, we show that the proposed approach is effective for detecting mislabeled data (Section~\ref{sec:exp_mislabeled}). Sifting through training examples that were responsible for making erroneous predictions can help in cleaning noisy data by removing malicious training points (Section~\ref{sec:expts_dataCleaning}). It can also assist in data summarization tasks to reduce storage costs and help human undertsanding (Section~\ref{sec:exp_summarization}). 
%We show that our method can be applied for training data summarization wherein coreset selection algorithms are typically used, and also for a curated fixing for mislabeled examples which was recently handled using influence functions. We present quantitative evaluations and show strong performance of our method on these tasks vs established baselines. 

Our key contributions are as follows:
\begin{itemize}\setlength\itemsep{-0.5em}
\item We propose a novel method to select salient training data points that explain test set predictions for black box models.
\item To solve the  resulting combinatorial problem, we develop new faster convergence guarantees for greedy Sequential Bayesian Quadrature on discrete candidate sets. One novel insight that results is the applicability of more scalable algorithm variants for SBQ with provable bounds. These theoretical insights may be of independent interest.
\item We recover the influence function based approach of~\citet{KohL17} as a special case. This connection again yields several novel insights about using influence functions for model interpretation and training side adversarial attacks. Most importantly, we establish the importance of the Fisher space for robust learning that can hopefully lead to promising future research directions.
\item To highlight the practical impact of the our interpretability framework, we present its application to three different real world use-cases. 
\end{itemize}

\textbf{Related work}: There has been a lot of interest lately in model interpretation in various ways and their corresponding applications. Thus, we focus our related work on the subset of most closely related research. Our approach has a similar motivation as ~\citet{KohL17}, who proposed the use of influence functions for finding the most \emph{influential} training data point for a test data point prediction. The intuition revolves around infinitesimally perturbing the training data point and evaluating the corresponding impact on the test point. The method is only designed for single data points -- thus their extension to selecting multiple data points required an unmotivated heuristic approach. A complementary line of research revolves around feature based interpretation of models. Instead of focusing on choosing representative data points, the goal is to reveal which features are important for the prediction~\citet{Ribeiro2016}. Recently,~\citet{Kim2016} also made use of the unweighted MMD function to propose selection of prototypes and criticisms. While their approach can be used for exploratory analysis of the data, it has not been extended for explaining a model. Their focus, moreover, is on the use of criticisms in addition to examples as a vital component of exploring datasets. %Development of algorithms for finding criticisms can be an interesting line of future work. 

Fisher kernels were proposed to exploit the implicit embedding of a structured object in a generative model for discriminative purposes~\citep{Jaakkola1999}, and have since been applied successfully in a variety of applications~\citep{Perronnin2010}. The goal is to design a kernel for generative models of structured objects that captures the ``similarity" for the said objects in the corresponding embedding space. The kernel itself can then be used out of the box in discriminative models such as Support vector machines. 

%% file: background.tex
\section{Background}
In this section, we provide an overview of the technical background required for our setup. We begin by fixing some notation. We represent sets using sans script fonts \textit{e.g.} $\sfA, \sfB$. Vectors are represented using lower case bold letters \textit{e.g.} $\bx,\by$, and matrices are represented using upper case bold letters \textit{e.g.} $\bX,\bY$.  Non-bold face letters are used for scalars \textit{e.g.} $j,M,r$ and function names \textit{e.g.} $f(\cdot)$. %The transpose of a vector or a matrix is represented by $(\cdot)^\top$ \textit{e.g.} $\bX^\top$. For singleton sets, we write $f(j) := f(\{j\})$. The cardinality of a set $\sfS$ is denoted by $|\sfS|$. For a vector $\bv$, the subvector supported on set of indices $\sfS$ is represented as $\bv_\sfS$. Similarly, the submatrix of $\bX$ supported on $\sfS$ is written as $\bX_{\sfS\sfS}$. We use the shorthand $[d]:=\{1,2,\ldots,d\}$. %We overload the notation slightly to write $\bX_\sfS$ as the submatrix with rows supported on $\sfS$ and all the columns.

\subsection{Fisher Kernels}
\label{sec:fisher}

The notion of similarity that Fisher kernels employ is that %two similar objects would have similar gradients in the parameters of the model. I
if two objects are \emph{structurally} similar, then slight perturbations in the neighborhood of the fitted parameters $\hat{\theta} := \arg\max \log p(\bX | \theta) $,  would impact the fit of the two objects similarly. In other words, the feature embedding $\bff_i := \frac{\partial \log p(\bX_i | \theta)}{\partial \theta} |_{\theta = \hat{\theta}}$, for an object $\bX_i \rightarrow \bff_i$ can be interpreted as a \emph{feature mapping} which can then be used to define a similarity kernel by a weighted dot product:
\[ \kappa(\bX_i, \bX_j) := \bff_i^\top \cI^{-1} \bff_j,
\]

where the matrix $\cI:= \bbE_{p(\bX)}[ \frac{\partial  \log p(\bX | \theta)}{\partial \theta}^\top\frac{\partial \log p(\bX | \theta)}{\partial \theta} ]$ is the Fisher information matrix. The information matrix serves to re-scale the dot product, and  is often taken as identity as it loses significance in limit~\citep{Jaakkola1999}. The corresponding kernel is then called the \emph{practical} Fisher kernel and is often used in practice. We note, however, that dropping $\cI$ had significant impact on performance in our method, so we employ the full kernel. However, the practical Fisher Kernel is important to mention here. As we show in Section~\ref{sec:influence}, using the practical Fisher Kernel recovers the influence function based approach to interpretability~\citep{KohL17} as a special case.  Another interpretation of the Fisher kernel is that it defines the inner product of the directions of gradient ascent over the Riemannian manifold that the generative model lies in~\citep{ShaweTaylor2004}.

%The use of appropriate feature mapping is crucial for predictive tasks. We observe that it is also vital for interpretability. 
While appropriate feature mapping is crucial for predictive tasks, we observe that it is also is vital for interpretability.
Fisher kernels are ideal for our task because they seamlessly extract model-induced data similarity from trained model that we wish to interpret. 
%for into the process of mapping features to a more amenable space. 
To further motivate that such a task can not be trivially performed by a something like a parameter sweep over RBF kernels i.e. without supervision, we perform a simple toy experiment illustrated in Figure~\ref{fig:motivate_fisher}.

\begin{figure*}
\centering
\includegraphics[scale=0.5]{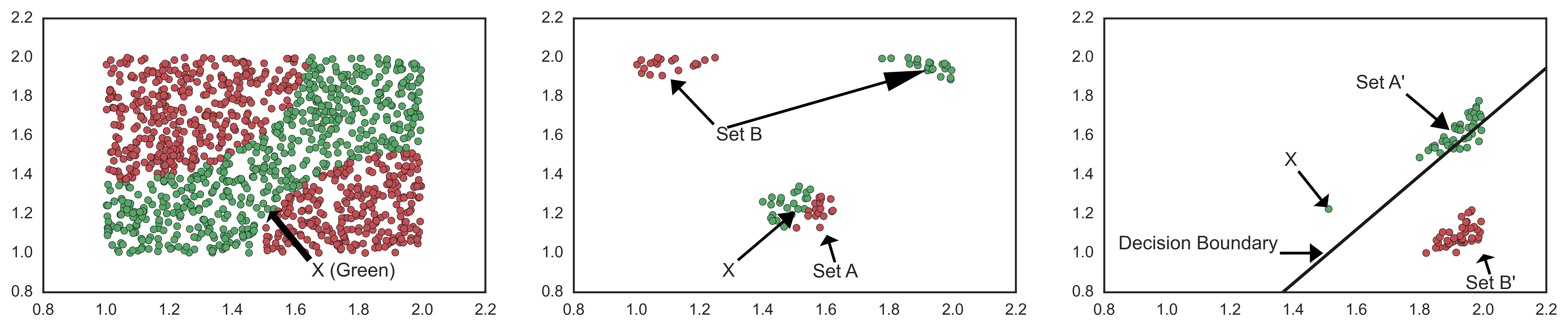}
\caption{A toy experiment to illustrate the usefulness of Fisher space mapping. [Left] 1200 samples on U[1,2] $\times$ U[1,2] with two labels - Green and Red as illustrated. A specific green point X is selected for further experiment. [Mid] Closest 40 (Set A) and farthest 40 points (Set B) in terms of RBF kernel similarity. A distance based kernel such as RBF would yield these points as most and least similar to X respectively. [Right] Closest 40 (Set A') and farthest 40 (Set B') to X in terms of the Fisher kernel similarity computed from a fitted logistic regression model. The decision boundary for the logistic regression is also presented. It predicts everything below it as red, and everything above it as green. The Fisher ``closeness'' here takes into account the label of the points as well as the log-likelihood gradient on the contour of the loss function and its direction for each point. Note that for points exactly on the boundary, their gradient and Fisher similarity with all other points will be $0$.}
\label{fig:motivate_fisher}
\end{figure*}

\subsection{Bayesian Quadrature}

Bayesian quadrature~\citep{Ohagan1991} is a method used to approximate the expectation of a function by a weighted sum of a few evaluations of the said function. Say a function $f:\cX\rightarrow \bbR$ is defined on a measurable space $\cX \subset \bbR^d$. Consider the integral:
\begin{equation}
 \bbE[f(\bx)] = \int_\cX f(\bx)p(\bx)d\bx \approx \sum_{i=1}^n w_i f(\bx_i),
 \label{eq:approxexpectation}
\end{equation}

where $w_i$ are the weights associated with function evaluations at $\bx_i$. Using $w_i = \nicefrac{1}{n}$ and randomly sampling $\bx_i$ recovers the standard Monte Carlo integration. Other methods include kernel herding~\citep{Chen2010SuperSamples} and quasi-Monte carlo~\citep{Dick2010book}, both of which use $w_i=\nicefrac{1}{n}$ but use specific schemes to draw $\bx_i$. Bayesian quadrature allows one to consider a non-uniform $w_i$ given a functional prior for $f(\cdot)$.  The samples $\bx_i$ can then be chosen as the ones that minimize the posterior variance~\citep{Huszar2012OptimallyWeightedHI} as we shall see in the sequel. The corresponding weights can be calculated directly from the posterior mean. We impose a Gaussian Process prior on the function as $f\sim \text{GP}(0,k)$ with a kernel function $k(\cdot,\cdot)$. The algorithm SBQ proceeds as follows. Say we have already chosen $n$ points: $\bx_i, i \in [n]$. The posterior of $f$ given the evaluations $f(\bx_i)$ has the mean function: 

\[\hat{f}(\bx) = \bk^\top \bK^{-1} \bff,\]

where $\bff$ is the vector of function evaluations $f(\bx_i)$, $\bk$ is the vector of kernel evaluations $k(\bx, \bx_i)$, and $\bK$ is the kernel matrix with $\bK_{ij} := k(\bx_i,\bx_j)$. 

We now focus on sampling the points $\bx_i$. The quadrature estimate provides not only the mean, but the full distribution as its posterior. The posterior variance can be written as:

\[\text{cov}(\bx,\by) = k(\bx,\by) -  k(\bx, \bX) \bK^{-1} k(\bX, \by), \]

where $\bX$ is the matrix formed by stacking $\bx_i$, and the kernel function notation is overloaded so that $k(\bX, \by)$ represents the column vector obtained by stacking $k(\bx_i, \by)$. The posterior over the function $f$ also yields a posterior over the expectation over $f$ defined in~\eqref{eq:approxexpectation}. For convenience, define the set $\sfS_j:=\{\bx_1, \bx_2,\ldots, \bx_j\}$. Say $Z(\sfS_j):= \sum_j w_j f(\bx_j)$. Then, it is straightforward to see $\bbE[Z(\sfS_n)] = \bz^\top \bK^{-1} \bff $, where $\bz_i  := \int k(\bx, \bx_i) p(\bx) d\bx$. Note that the weights in~\eqref{eq:approxexpectation} can be written as $w_i = \sum_j \bz_j [\bK^{-1}]_{ij} $. 

We can write the variance of $Z(\sfS_n)$ as:

\begin{equation}
\label{eq:sbq_costfunction}
\text{var}(Z(\sfS_n))  = \iint  k(\bx,\by) p(\bx)p(\by) d\bx d\by - \bz^\top \bK^{-1} \bz. 
\end{equation}

The algorithm Sequential Bayesian Quadrature (SBQ) samples for the points $\bx_i$ in a greedy fashion with the goal of minimizing the posterior variance of the computed approximate integral:
\[ \bx_{n+1} \leftarrow \argmin_{\bx \in \cX}\text{var}(Z(\sfS_{n} \cup \{\bx\})). \]

 %{\color{magenta} talk about rates for greedy selection based on experiments}

%Let $\cH$ be a Reproducing Kernel Hilbert Space (RKHS) of functions, and $f(\cdot) \in \cH$.

%% file: fisher.tex
\section{Prototype Selection using Fisher Kernels}
In this section, we present our method to select sample representatives using Fisher kernels. For a loss function $\ell(\theta, \bx)$, where $\theta$ are the parameters of the model and $\bx$ is the data, to train a parametric model one would minimize the expected loss:
\begin{equation}
\label{eq:expectedloss}
\min \bbE_{p(\bx)} \ell(\theta, \bx),
\end{equation}

where $p(\bx)$ is the data distribution. Since we usually do not have access to the true data distribution, $p(\bx)$ is typically the empirical data distribution $p(\bx)= \frac{1}{n}\delta(\bx)$, where $\delta(\cdot)$ is $1$ if $\bx$ exists in the dataset, and $0$ otherwise, and $n$ is the size of the dataset. Our goal in this work is to  approximate the integral~\eqref{eq:expectedloss} over the test or validation set (which specifies the distribution $p$ for us) using a weighted sum of a \emph{few} points from the \emph{training} dataset~\eqref{eq:approxexpectation}. Note that while the training samples in general have measure 0 in the test or the validation set distribution in the euclidean space, the smoothening GP prior over the embedding space still allows for samples to be generated from the former to approximate the latter. 

For the kernel function in the GP prior in Bayesian Quadrature, we use the Fisher kernel of the trained parametric model. SBQ selection strategy inherently establishes a trade off between selecting data points that are representative of the parametric fit and diversity of the selected points. To see this, consider the SBQ cost function~\eqref{eq:sbq_costfunction}. At every new selection $\bx_{j+1}$, one one hand, the cost function rewards the selection of data points which are clustered closer together in the feature mapping space to increase the value of $\bz$ which in turn decreases variance. However, on the other hand, selecting points close to each other decreases the eigenvalues of $\bK^{-1}$ thereby increasing variance~\citep{Huszar2012OptimallyWeightedHI}. Thus, the SBQ seeks a tradeoff between these terms.

\subsection{An Efficient Greedy Algorithm}
In this section, we provide a practical greedy algorithm to select representative prototypes using SBQ to optimize~\eqref{eq:sbq_costfunction}. Note that the first term is constant w.r.t to $\sfS_n$. Moreover, $p(\bx) = \frac{1}{n}\delta(\bx)$. Thus, we can re-write $\bz_i = \frac{1}{n}\sum_{j=1}^n k(\bx_i, \bx_j)$ for each $i$ in training and each $j$ in the test set. This can be pre-computed by a row or column sum over the kernel of the entire dataset in $O(nt)$ time and stored as vector of size $t$ to speed up later computation, where $t$ is the size of the training set and $n$ is the size of the test set. Our greedy cost function at step $j+1$ is thus:% Morevoer, each $\bz_i$ can be calculated independently of others and in parallel over multiple cores or multiple machines with shared memory access. 
%Recall that $\bz_\sfS$ for a set $\sfS$ represents the subvector of $\bz$ with indices in $\sfS$, and similarly $\bK_{\sfS\sfS}$ represents the submatrix of $\bK$ with columns and rows indexed by $\sfS$.  
\begin{equation}
\label{eq:setcostfunction}
i^\star_{j+1} \leftarrow \arg\max_{\substack{i \in [m]\backslash\sfS_j \\ \sfS = \sfS_j \cup i  }} \bz_{\sfS}^\top [\bK_{\sfS\sfS}^{-1}] \bz_{\sfS}.
\end{equation}
The solution set is then updated as $\sfS_{j+1} = \sfS_j \cup \{i^\star_{j+1}\}$. The optimization~\eqref{eq:setcostfunction} requires an inverse of the kernel matrix of already selected data points which can be computationally expensive. However, we can use the following result from linear algebra about block matrix inverses to speed up operations. 

\begin{proposition}
\label{prop:matrixinverse} For an invertible matrix $\bA$, a column vector $\bb$, and a scalar $c$, let $d = c - \bb^\top \bA^{-1} \bb$, then
\[
\begin{bmatrix} \bA & \bb \\ 
\bb^\top &  c 
\end{bmatrix}^{-1} =   
\frac{1}{d}\begin{bmatrix}
d\bA^{-1} + \bA^{-1} \bb\bb^\top \bA^{-1}  & \bA^{-1}\bb \\ \bb^\top\bA^{-1}  & 1 
\end{bmatrix} \]
\end{proposition}

Proposition~\ref{prop:matrixinverse} allows us to build the inverse of the kernel $\bK$ in~\eqref{eq:setcostfunction} greedily. The full algorithm is presented in Algorithm~\ref{algo:greedy}. 

\begin{algorithm}[h]
\caption{Greedy Prototype Selection}
\label{algo:greedy}
\centering
\begin{algorithmic}[1]
\STATE  \textbf{INPUT:} Data $\{\bx_i\}$, kernel function $k(\cdot,\cdot)$, number of selections to make $k$
\STATE //Pre-compute $\bz_i$
\STATE $\bz_i = \frac{1}{n} \sum_j  k(\bx_i, \bx_j) \forall i \in $ training and $j \in$ test
\STATE // Build solution set $\sfS$ greedily. Maintain current inverse(K) at each iteration as $\invK$
\STATE $\sfS = \emptyset$, $\invK = []$
\FOR{$i=1\ldots k$}
\STATE $j^\star = -1$, $\text{MAX} = - \infty$
\FOR{$j \in [t]\backslash \sfS$}
\STATE $\bt = \bz[\sfS \cup j]$
\STATE  $\bb = k(\bX_{\sfS}, \bx_{j^\star})$, $c = k(\bx_{j^\star}, \bx_{j^\star})$, $\bA^{-1} = \invK$  Get $\bT$ as the updated inverse using  Prop.~\ref{prop:matrixinverse}
\STATE If $(\bt^\top\bT\bt > \text{MAX})$, $j^\star = j$, $\text{MAX} = \bt^\top\bT\bt$
\ENDFOR
%\STATE $j^\star = \arg\max_{ j \in [m]\backslash\sfS } \bz_j^\top \invK \bz_j$
\STATE Write $\bb = k(\bX_{\sfS}, \bx_{j^\star})$, $c = k(\bx_{j^\star}, \bx_{j^\star})$, $\bA^{-1} = \invK$
\STATE Update: $\invK$ using Prop.~\ref{prop:matrixinverse}, $\sfS = \sfS\cup j^\star$
\ENDFOR
\STATE return $\sfS$
\end{algorithmic}
\end{algorithm}

Algorithm~\ref{algo:greedy} obviates the need for taking explicit inverses and only requires an oracle access to the kernel function. The algorithm itself is inherently embarrassingly parallelizable over multiple cores. We study guarantees for the algorithm in Section~\ref{sec:weaksubmod} which also motivates its more scalable variants.
% Building the vector $\bz$ requires $O(nt)$ computation. For each $i\in [k]$, finding $j^\star$ requires a maximization over the remaining candidate vectors, and hence  $O(k^2t)$ computation (as opposed to $O(k^3t)$ when taking explicit inverses), and updating $\invK$ requires $O(kt)$ computation. The time taken to build the kernel matrix is $O(t^2)$.
%\subsection{Relationship with MMD}
%{\color{magenta} skip ?}

%% file: submod.tex
\section{Analysis}
\label{sec:weaksubmod}
The greedy algorithm described in Algorithm~\ref{algo:greedy} while being simple also has interesting optimization guarantees that make it attractive to use in practice. In this section, we provide convergence guarantees for the cost function~\eqref{eq:sbq_costfunction} as $n$ increases. Typically for functions like these in the general case, the candidate set of atoms used to build the approximation is uncountably infinite - any possible sample from the underlying density is a candidate. As such, the convergence results are based on using Frank-Wolfe analysis on the marginal polytope~\cite{Bach2012Herding}.  However, for us the underlying set of candidate atoms are discrete points, which are at worst countably infinite. As such, for this special case, it is worth analyzing if we can provide better rates than the general available guarantees. It turns out that this is indeed possible. We are able to leverage recent research in discrete optimization to indeed provide a linear convergence rate for the forward greedy algorithm.  
 
Recall our set optimization function (from~\eqref{eq:setcostfunction}) is:

\begin{equation}
\label{eq:setoptimizationproblem}
 g(\sfS) := \max_{\substack {\sfS \subset [n] \\ | \sfS| \leq r}} \bz_\sfS [\bK_{\sfS\sfS}^{-1}] \bz_\sfS,
\end{equation}

where $n$ is the set of candidate training data points. We write $\mu_p :=\iint  k(\bx,\by) p(\bx)p(\by) d\bx d\by $. For the RKHS induced by the kernel $\cH$, we can equivalently re-write the cost function as~\citep{Huszar2012OptimallyWeightedHI,Bach2012Herding}: 

\begin{equation}
\label{eq:leastsquareoptimizationproblem}
\min_{\substack {\sfS \subset [n] \\ | \sfS| \leq r}} v(\sfS):=  \mu_p - \sum_{i \in \sfS} w_i \bz_i 
\end{equation}

For a matrix $\bA$, the smallest (largest) $k$-sparse eigenvalues is min (max) of $\frac{\bx^\top \bA \bx}{\bx^\top \bx}$ under the constraints $\| \bx\|_0 \leq k$, and $\bx \neq 0$. Note that we can write $v(\emptyset)= \mu_p  $. We present our convergence guarantee next. 

\begin{theorem}
\label{thm:convergence}
Say $\cH$ is finite dimensional and has bounded norm i.e. $\forall \nu \in \cH$, $\| \nu\|_\cH  < \infty$. Let $m$ be the smallest $2r$ sparse eigenvalue and $M$ be the largest $r+1$-sparse eigenvalues of the kernel matrix  $\bK$ of the training set. If $\sfS_G$ of size $k$ is the set returned by Algorithm~\ref{algo:greedy} and $\sfS^\star$ of size $r$ is the optimal solution of~\eqref{eq:leastsquareoptimizationproblem}, then if $k \geq \frac{M}{m}r \log \frac{1}{\epsilon}$, $v(\sfS_G)  - v(\sfS^\star) \leq \epsilon (v(\emptyset)  - v(\sfS^\star))$.
\end{theorem}

\paragraph{Discussion:} Theorem~\ref{thm:convergence} provides exponential convergence for the cost function $v(\cdot)$. For the same objective, using Frank-Wolfe on the marginal polytope, the best known guarantees in the most general case are $O(1/T)$ for finite dimensional bounded Hilbert spaces~\citep{Bach2012Herding}. In the special case when the optimum lies in the relative interior, we do get faster exponential convergence. Theorem~\ref{thm:convergence} provides an alternative condition that is sufficient for exponential convergence for the case when the optimum $\mu_p$ lies at the boundary of the marginal polytope instead of in its relative interior i.e. it is linear combination of $r$ atoms. The lower sparse eigenvalue condition is a union bound, and only requires to hold over the greedy selection set plus any $r$ sized subset.

\subsection{Scalability} 
For massively large real world datasets, the standard greedy algorithm (SBQ) may be prohibitively slow. In addition to run time, there are also memory considerations. SBQ requires building and storing an $O(m^2)$ sized kernel matrix over the training set of size $m$. We can use alternative variants of the greedy algorithm that are either faster with some compromise on the convergence rate or can distribute the kernel over multiple machines. These variants are presented in Table~\ref{tab:algorithms} with their corresponding references. To the best of our knowledge, these variants have not been suggested for solving the problem~\eqref{eq:approxexpectation} before and may be of independent interest. The convergence rates are obtained similar to the proof of Theorem~\ref{thm:convergence} by plugging in respective approximation guarantees in lieu of Lemma~\ref{lem:greedy_weaksubmod_guarantee} in the appendix.

\begin{table*}
\centering
\begin{tabular}{|c|c|c|c|}\hline
Algorithm & Runtime & Memory required & Convergence rate \\\hline
SBQ (Algorithm~\ref{algo:greedy})  & $O(k^3t)$ & $O(t^2 + n)$ &  $O(\lambda  log \nicefrac{1}{\epsilon})$ \\ 
Matching Pursuit~\citep{Elenberg2016}  & $O(k^2t)$ & $O(t^2 + n)$ &  $O(\lambda  log \nicefrac{1}{\epsilon})$ \\ 
$\delta$-Stochastic Selection~\citep{Khanna2017}  & $O(kt\log \nicefrac{1}{\delta})$ & $O(t^2 + n)$ &  $O(\lambda log \nicefrac{1}{(\delta\epsilon)})$ \\ 
Distributed ($l$ machines)~\citep{Khanna2017}  & $O(\frac{kt}{l})$ & $O(\frac{t^2}{l^2} + n)$ &  $O(\lambda  log \frac{1}{\epsilon})$ \\ \hline
\end{tabular}
\caption{Greedy variants for prototype selection. $\lambda=\frac{M}{m}$, $n$ is the test set size, $t$ is the size of the training set. Convergence rate refers to number of iterations needed to get $\epsilon$ accuracy. For Stochastic and Distributed variants, the guarantee is in expectation.}
\label{tab:algorithms}
\end{table*}

%% file: influence.tex
\section{Relationship with Influence functions}
\label{sec:influence}
Influence functions~\citep{Cook1980Influence} have recently been proposed as a tool for interpreting model predictions~\citep{KohL17}. Since our goal is also the same, it is interesting to ask if there is a relationship between the two approaches. For selecting the most influential training point for a given test point, influence functions approximate infinitesimal upweighting of which training point has the most effect on prediction of the test point in question. In this section, we show that our method recovers this influence function approach used by~\citet{KohL17} for selecting influential training data points. In addition, we also show how adversarial training side attacks proposed by~\citet{KohL17} by perturbing features of training data points can be re-interpreted as a standard adversarial attack in the RKHS induced by the Fisher Kernel. Our analysis yields new insights about the influence function based approach and also establishes the importance of the Fisher space for robust learning. 

\subsection{Choosing training data points}
\label{sec:influence_datapoints}
We briefly introduce the influence function approach for model interpretation. For simplicity, we re-use the notation suggested by~\citet{KohL17}. Let $\ztest$ be the test data point in question, $\sfS_\text{train}$ be the training set, $L(\bz, \theta)$ be the loss function fitted on the training set, $\hat{\theta}$ be the optimizer of $L(\sfS_\text{train},\theta)$, $\bH_\theta$ be the Hessian of the loss function evaluated at $\theta$, then the most influential training data point is the solution of the optimization problem:

\begin{equation}
\label{eq:influence_costfn}
\max_{\bz \in \sfS_\text{train}} \nabla_\theta L(\bz, \hat{\theta}) \bH_{\hat{\theta}}^{-1} \nabla_\theta L(\ztest, \hat{\theta}) 
\end{equation}

We compare the two discrete optimization problems~\eqref{eq:setoptimizationproblem} and ~\eqref{eq:influence_costfn}. Even though~\eqref{eq:setoptimizationproblem} uses first order information only while~\eqref{eq:influence_costfn} uses both first order and second order information about the loss function, the following proposition illustrates a connection.

\begin{proposition}
\label{prop:influence}
If the loss function $L(\cdot)$ takes the form of a negative log-likelihood function, $[H_{\hat{\theta}}]_{ij} = \nabla_{\theta_i}  L(\sfS_{\text{train}}, \hat{\theta})^\top  \nabla_{\theta_j} L(\sfS_{\text{train}}, \hat{\theta})$, where we have overloaded the notation $L(\sfS_{\text{train}}, \theta) = \frac{1}{| \sfS_{\text{train}}|  } \sum_t L(\bz_t, \theta )$. 
\end{proposition}
\begin{proof}
Let $L(\bz, \theta) := - \log p(\bz, \theta)$, since it takes form of a negative LL function. Then, since $\hat{\theta}$ is the optimizer of $L(\sfS_{\text{train}}, \theta) $,
\begin{align*}
 & \nabla_{\theta_i} L(\sfS_{\text{train}}, \hat{\theta}) = 0 \\
\implies & \nabla_{\theta_i} \sum_t - \log p (\bz_t, \hat{\theta}) = 0 \\
\implies & \nabla_{\theta_j} \nabla_{\theta_i} \sum_t - \log p (\bz_t, \hat{\theta}) = 0 \\
\implies & \nabla_{\theta_j} \sum_t \frac{-1}{p (\bz_t, \hat{\theta}) } \nabla_{\theta_i}  p (\bz_t, \hat{\theta})  = 0 \\
\implies & \sum_t \frac{\nabla_{\theta_i} \nabla_{\theta_j} p (\bz_t, \hat{\theta}) }{p (\bz_t, \hat{\theta})} = \sum_t \frac{\nabla_{\theta_i}p (\bz_t, \hat{\theta})^\top \nabla_{\theta_j} p (\bz_t, \hat{\theta}) }{p (\bz_t, \hat{\theta})^2},
\end{align*}
from which the result directly follows. 
\end{proof}

From Proposition 1, it is easy to see that the optimization problems~\eqref{eq:setoptimizationproblem} and~\eqref{eq:influence_costfn} are the same under some conditions. To be more precise, we can make the following statement. If the cost function $L(\cdot,\cdot)$ is in the form of a negative log-likelihood function,~\eqref{eq:influence_costfn} is a special case of~\eqref{eq:setoptimizationproblem} with the practical Fisher kernel (see Section~\ref{sec:fisher}) when the test set is of size 1, and $r=1$. 

This equivalence gives several insights about influence functions that were not known before: (1) it generalizes influence functions to multiple data points for both test and training sets in a principled way and provides a probabilistic foundation to the method, (2) it establishes the importance of the induced RKHS by the Fisher kernel by re-interpreting the influence function optimization problem as $\min_{\bz \in \sfS_{\text{train}} } \| \ztest - \bz \|_\cH$ (see Lemma~\ref{lem:mmd_bq} in the appendix), (3) for negative LL functions, it renders the expensive the calculation of the Hessian in the work by~\citet{KohL17} as redundant since by Proposition~\ref{prop:influence}, first order information suffices, (4) it provides theoretical approximation guarantees (see Lemma~\ref{lem:greedy_weaksubmod_guarantee} in the appendix) for selection of multiple training data points, in constrast to ~\citet{KohL17} who made multiple selections greedily only as a heuristic.  

\subsection{Unified view of adversarial attacks}
Given a test data point $\bz$, an adversarial example is generated by adding a small perturbation as $\tilde{\bz} = \bz + \epsilon_z$, where $\epsilon_z$ is a small perturbation of $\bz$ so that for $\tilde{\bz}$ is indistinguishable from $\bz$ by a human, but causes the model to make an incorrect prediction on $\bz$~\citep{Goodfellow2014ExplainingAH}. For training data attacks, $\bz$ is a training data point that is perturbed to make an incorrect prediction on a test data point. For a loss function $\ell(\bz)$, a test side attack for perturbing a test data point $\ztest$ would solve the optimization problem:

\begin{equation}
\label{eq:testsideattack}
\max_{\| \bz - \ztest\|_\infty \leq \epsilon} \ell(\bz)
\end{equation}

While the optimization~\eqref{eq:testsideattack} is hard in general, typically a few iterations of projected gradient ascent or FGSM are applied. We refer to the recent work by~\citet{Madry2018Towards} for details. 

For training side attacks,~\citet{KohL17} perform the following iterative update: 

\begin{equation}
\label{eq:trainsideattack_influencefn}
\tilde{\bz} \leftarrow  \Pi(\tilde{\bz} + \alpha \text{sign} (\cL(\tilde{\bz}, \ztest))),
\end{equation}

where $\bz=(x,y)$ is a candidate training example to perturb in $x$, $\ztest$ is the target test example, $\Pi$ is the projection operator onto the set of valid images, $\alpha$ is a fixed step size, and $\cL(\tilde{\bz}, \ztest):= \nabla_\theta L(\bz, \hat{\theta}) \bH_{\hat{\theta}}^{-1} \nabla_x \nabla_\theta L(\ztest, \hat{\theta}) $. 

Using the results in Section~\ref{sec:influence_datapoints}, it is straightforward to see that the if we use $\ell(\bz) = \| \ztest - \bz \|_\cH$, where $\cH$ is the RKHS induced by the practical Fisher kernel, and change the constraint as a perturbation over a training example instead of the test example, we recover the iterative step~\eqref{eq:trainsideattack_influencefn} as a special case of projected gradient ascent steps to solve~\eqref{eq:testsideattack}. 

This equivalence provides a unified view of both training and test side attacks. As such, the large literature on robust learning against test side attacks can be applied to robustness against training side attacks as well. Moreover our framework also provides a principled way to do training side attacks to target multiple test set examples, instead of attacking individual test points separately.
% Sanmi: This is super cool! I wish we had time to explore more. Almost makes me want to wait for ICML to flesh this out, and possibly more experiments. Oh well...

%% file: experiments.tex
\section{Experiments}
\label{sec:experiments}
%4s and 9s are harder - accuracy of 989955 vs 0.9922. chcked 
We present empirical use cases of our framework. We chose the experiments to illustrate the flexibility of our framework, as well as to emphasize its generalization capacity over and above influence functions. As such, we present experiments that make use of set influence (as opposed to single data point influence) for data cleaning and summarization (Sections~\ref{sec:expts_dataCleaning},\ref{sec:exp_summarization}). To illustrate potential benefit of using the full Fisher kernel as opposed to the simplified practical Fisher kernel as used by the influence functions, we present evaluation for a use case for fixing mislabelled examples as presented by~\citet{KohL17} (Section~\ref{sec:exp_mislabeled}).

\subsection{Data Cleaning: removing malicious training data points}
\label{sec:expts_dataCleaning}
In this section, we present experiments on the MNIST dataset to illustrate the effectiveness of our method  in interpreting model behavior for the test population. Some of the handwritten digits in MNIST  are hard even for a human to classify correctly. Such points can adversely affect the training of the classifier, leading to lower predictive accuracy. Our goal in this experiment is to try to identify some such misleading training data points, and remove them to see if it improves predictive accuracy. To illustrate the flexibility of our approach, we focus only on the digits $4$ and $9$ in the test data which were misclassified by our model, and then select the training data points responsible for those misclassifications. 

The MNIST data set~\citep{lecunMNIST} consists of images of handwritten digits and their respective labels. Each image is a $28\times 28$ pixel array. There are $70000$ images in total, split into $60000$ training examples and $10000$ test examples. The $10$ digits are about evenly represented in both the training and the test data. 

For the classification task, we use tensorflow~\citep{tensorflow} to build a 2 layer convolutional network with $2\times 2$ max pooling followed by a fully connected layer and the softmax layer. The convolutions use a stride of $1$ followed by padding of zeros to match the input size. We use dropout to avoid overfitting. The network was trained using the built-in Adam Optimizer for $20000$ steps of batch size $100$ each. For the entire test set, we obtain an accuracy of $0.9922$, while for the subset of the test set consisting only of the chosen two digits $4$ and $9$, the accuracy is $0.9889$. 

 %The last layer of the network can be thought of as a proxy embedding of the image features, and the respective gradients can thus be used for calculating the Fisher kernels. Note that we could also use the gradients over all trainable variables, but the current implementation of tensorflow and kernel function libraries we used from \texttt{sklearn} make it prohibitively expensive. 

\begin{figure}
\centering

\begin{subfigure}[t]{0.49\textwidth}
\centering
\includegraphics[scale=0.35]{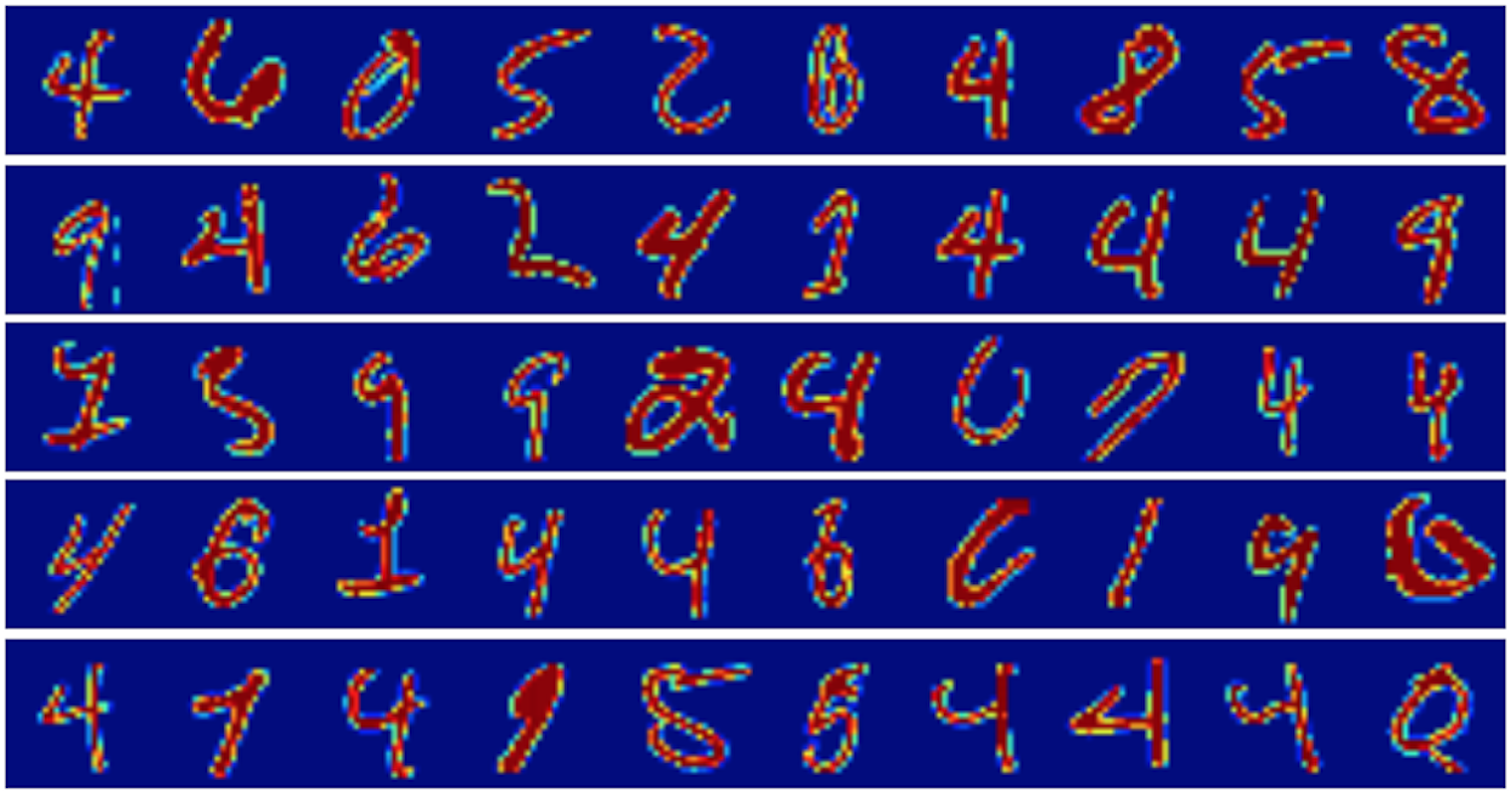}
\caption{A subset of selected prototypes responsible for misclassifying $4$s and $9$s in the test set}
\label{fig:subfig_curationset}
\end{subfigure}

\begin{subfigure}[t]{0.45\textwidth}
\centering
\includegraphics[scale=0.38]{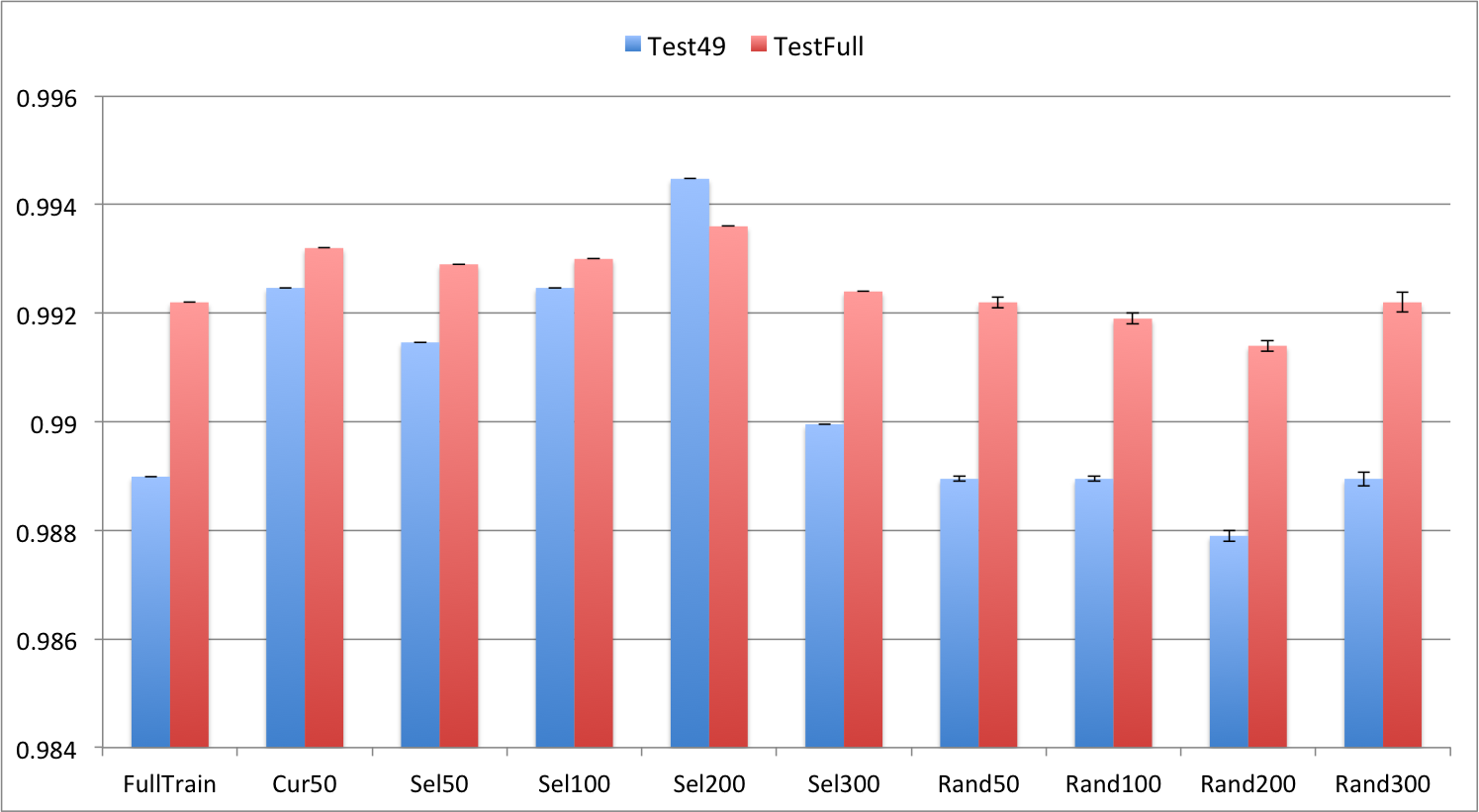}
\caption{Accuracy fractions on test data $4$s and $9$s (Test49), and the full test set after removing random (Rand), algorithm selected (Sel), or Curated (Cur) prototypes.  }
\label{fig:subfig_mnistaccuracy}
\end{subfigure}
\caption{MNIST experiment for selecting malicious training data points.}
\label{fig:mnistexperiment}
%\vspace{-1mm}
\end{figure}

After the training is completed, we obtain the gradients of the training and test data points w.r.t the parameters of the network by passing  each point through the trained (and subsequently frozen) network. The obtained gradient vectors are used to calculate the Fisher kernel as detailed in Section~\ref{sec:fisher}. We then employ Algorithm~\ref{algo:greedy} using the newly built Fisher kernel matrix between training and test datasets to obtain the top $300$ \emph{prototypes} i.e. data points from the training set that our algorithm deems most responsible for misclassifying $4$s and $9$s. 

To check if these points are indeed misguiding the model, we remove the top $50,100,200,300$ of the selected points from the training data and retrain the model to retest on the test set. These numbers are reported as Sel50, Sel100, Sel200, Sel300 in Figure~\ref{fig:subfig_mnistaccuracy}. Indeed we see an improvement in the test accuracy till Sel200 indicating the importance of removing the selected potential malicious points from the training set, and a subsequent decay in performance for Sel300 most likely due to removal of too many useful points in addition to malicious ones. To compare, we also remove the respective number of points randomly and repeat the experiment. Removal of random points from the training data led to a general decay in the predictive accuracy. 

Finally, we manually selected $50$ points from the chosen $300$ points as the curated set based on how ill-formed the digits were (see Figure~\ref{fig:subfig_curationset}). Removing these points from the training set before re-training and testing gives predictive accuracy is reported as Cur50 comparable to Sel100, but still worse than Sel200, indicating that the algorithm identified more malicious points in top-200 selected than our manually chosen $50$ points. 
%
%\begin{figure}
%%\centering
%\includegraphics[scale=0.24]{figures/mis49}
%%\includegraphics[height=200px, width=240px]{figures/mis49}
%\hfill
%\includegraphics[scale=0.29]{figures/columnchart}
%\caption{[Left] A subset of selected prototypes responsible for misclassifying $4$s and $9$s in the test set. [Right] Accuracy fractions on Test set $4$ and $9$ (Test49), and the full test set after removing random (Rand), algorithm selected (Sel), or Curated (Cur) prototypes.}
%\label{fig:mnistexperiment}
%\end{figure}

\subsection{Fixing Mislabeled Examples}
\label{sec:exp_mislabeled}

\begin{figure}%{r}{.65\textwidth}%[ht]
\centering
\includegraphics[scale=0.58]{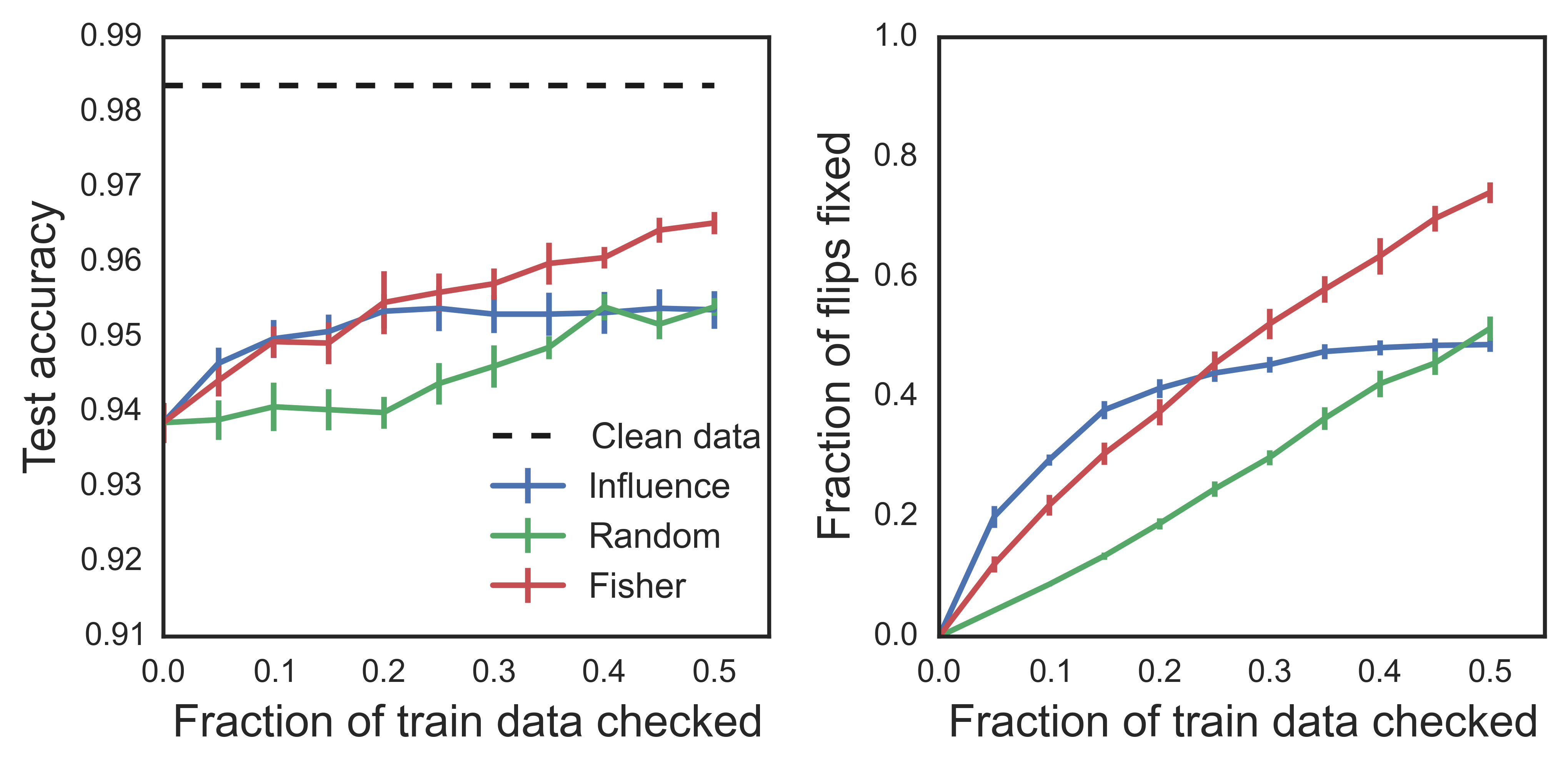}
\caption{Comparison of SBQ compared to Influence functions on the task of fixing flipped labels.}
 \label{fig:spam}
\end{figure}

%
%\begin{figure*}
%\centering
%\begin{subfigure}[t]{0.45\textwidth}
% \includegraphics[scale=0.34]{figures/chemreact}
% \end{subfigure}
% \begin{subfigure}[t]{0.45\textwidth}
% \includegraphics[scale=0.34]{figures/covtype}
% \end{subfigure}
% \caption{Performance for logistic regression over two datasets (left is \texttt{ChemReact} while right is \texttt{CovType}) of our method (Fisher) vs coreset selection~\citep{HugginsCB16} and random data selection. `Full' reports the numbers for training with the entire training set. Fisher achieves much better test LL performance than the baselines over several different subset sizes.}
% \label{fig:coresets}
% %\vspace{-4mm}
%\end{figure*}
%% \vspace{-3mm}

In this experiment, we use our framework to detect and fix mislabeled examples. Labor intenstive labeling tasks naturally result in mislabels, especially in real-world datasets. These data points may cause poor performance and degradation of the model. We show that our method can be successfully used for this purpose, showing improvement over the recent results by~\citet{KohL17}. 
%In this experiment, we focus on datasets which could have binary labels some of their data points flipped. 

%It is desirable to develop methods that can reduce the involved human effort while still being able to identify mislabeled data points.

%We can perform this task using our framework. This is 
% BEEN: [TODO:HOW MANY] . Rajiv: It is mentioned below in the specific setting of the enron data.
We use a small correctly labeled validation set to identify examples from the large training set that are likely mislabeled. 
We first train a classifier on the noisy training set, and predict on the validation set. We then employ Algorithm~\ref{algo:greedy} to identify training examples that were responsible for making incorrect predictions on the validation set.
The potentially mislabeled data points are then chosen by the output of our method.
%The selected training data points are our output as the points we believe are wrongly labeled. 
Curation is then simulated on the selected examples in order of selections made (similar to the approach by~\citet{KohL17}), and if the label was indeed wrong, it is fixed. We report on the number of training data points selected vs fixed (the precision metric for incorrectly labeled points) and the respective improvement in unseen test data accuracy.

%Influence functions are similar to kernel herding and quadrature selection, in the sense that they explain prediction on a test point by selecting a training point that has the largest \emph{influence}, where the influence is measured by change in the test error by perturbing the training point. Since the motivation is different, so is the cost function. Influence function evaluations require second order calculations, while our method is first order requiring only gradient evaluations. Secondly, we have a principled way to dealing with \emph{sets}, as opposed to single train/test points. 

For evaluation, we use \texttt{enron1} email spam dataset used by~\citet{KohL17} and compare our results to their reported results.
%We re-use the experiment done by~\citet{KohL17} to compare our method. We use the same \texttt{enron1} email spam dataset. 
The dataset consists of $4137$ training points and $1035$ test points. We randomly select $500$ data points from the training set as the clean \emph{curated} data. From the remaining training data points, we randomly flip the labels of $20\%$ of the data. We then use our method and the baselines to select several candidates for curation. We report the number of fixes made after these selections and the corresponding test predictive accuracy. The baselines are selection by (i) top self influence measures~\citep{KohL17}, and (ii) random selection of datapoints. The curation data is used as part of the training by all the methods. No method had access to the test data. As showing in Figure~\ref{fig:spam}, our algorithm consistently performs better in test accuracy and the fraction of flips fixed as more and more data is curated.

%[SAY SOME PUNCHLINE HERE AND SUMMARY OR ADD A SEGWAY TO NEXT EXP SECTION. Something like now we did mislabed, can we use this as a data cleaning method to cleanup data points that are not as bad as mislabeled points, but still bad for the model performance?]

\subsection{Data Summarization}
\label{sec:exp_summarization}
In this section, we perform the task of training data summarization. Our goal is to select a few data samples that represent the data distribution \emph{sufficiently} well, so that a model built on the selected subsample of the training data does not degrade too much in performance on the unseen test data. This task is complimentary to the task of interpretation, wherein one is interested in selecting training samples that explain some particular predictions on the test set. Since we are interested in approximating the test distribution using a few samples from a training set with the goal of predictive accuracy under a given model, our framework of Sequential Bayesian Quadrature using Fisher kernels is directly applicable. 

Another method that also aims to do training data summarization is that of coreset selection~\cite{HugginsCB16}, albeit with a different goal of reducing the training data size for optimization speedup while still maintaining guaranteed approximation to the training likelihood. Since the goal itself is optimization speedup, coreset selection algorithms typically employ fast methods while still trying to capture the data distribution by proxy of the training likelihood. Moreover, the coreset selection algorithm is usually closely tied with the respective model as opposed to being a model-agnostic method like ours. 

To illustrate that coreset selection falls short on the goal of competitively estimating the data distribution, we employ our framework to the problem of training data summarization under logistic regression, as considered by~\citet{HugginsCB16} using coreset construction. We experiment using two datasets \texttt{ChemReact} and \texttt{CovType}. \texttt{ChemReact} consists of $26733$ chemicals each of feature size $100$. Out of these,  $2500$ are test data points. The prediction variable is $0/1$ and signifies if a chemical is reactive. \texttt{CovType} has $581012$ webpages each of feature size $54$. Out of these, $29000$ are test points. The task is to predict whether a type of tree is present in each location or not. 

In each of the datasets, we further randomly split the training data into $10\%$ validation and $90\%$ training. For the larger \texttt{CovType} data, we note that selecting about 20,000 training points out of the training set achieves about the same  performance as the full set. Hence, we work with randomly selected 20,000 points for speedup. We train the logistic regression model on the new training data, and use the validation set as a proxy to the unseen test set. We build the kernel matrix $\bK$ and the affinity vector $\bz$, and run Algorithm~\ref{algo:greedy} for various values of $k$. For the baselines, we use the coreset selection algorithm and  random data selection as implemented by~\citet{HugginsCB16}. The results are presented in Figure~\ref{fig:coresets}. We note that our algorithm yields a significantly better predictive performance compared to random subsets and coresets~\cite{HugginsCB16} with the same size of the training subset across different subset sizes.

\begin{figure}
\centering
\begin{subfigure}[t]{0.45\textwidth}
\centering
 \includegraphics[scale=0.3]{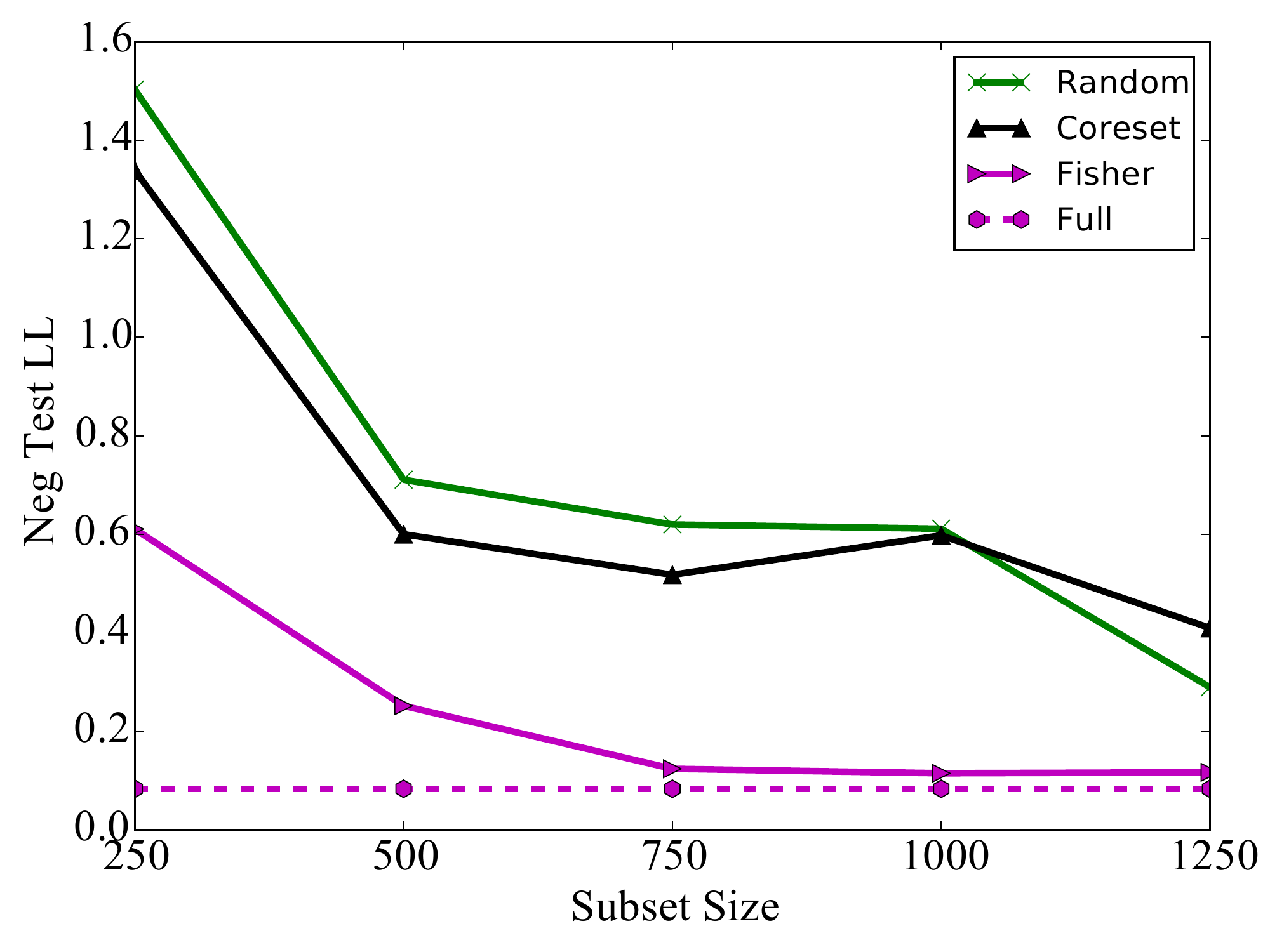}
 \end{subfigure}
 \begin{subfigure}[t]{0.45\textwidth}
 \centering
 \includegraphics[scale=0.3]{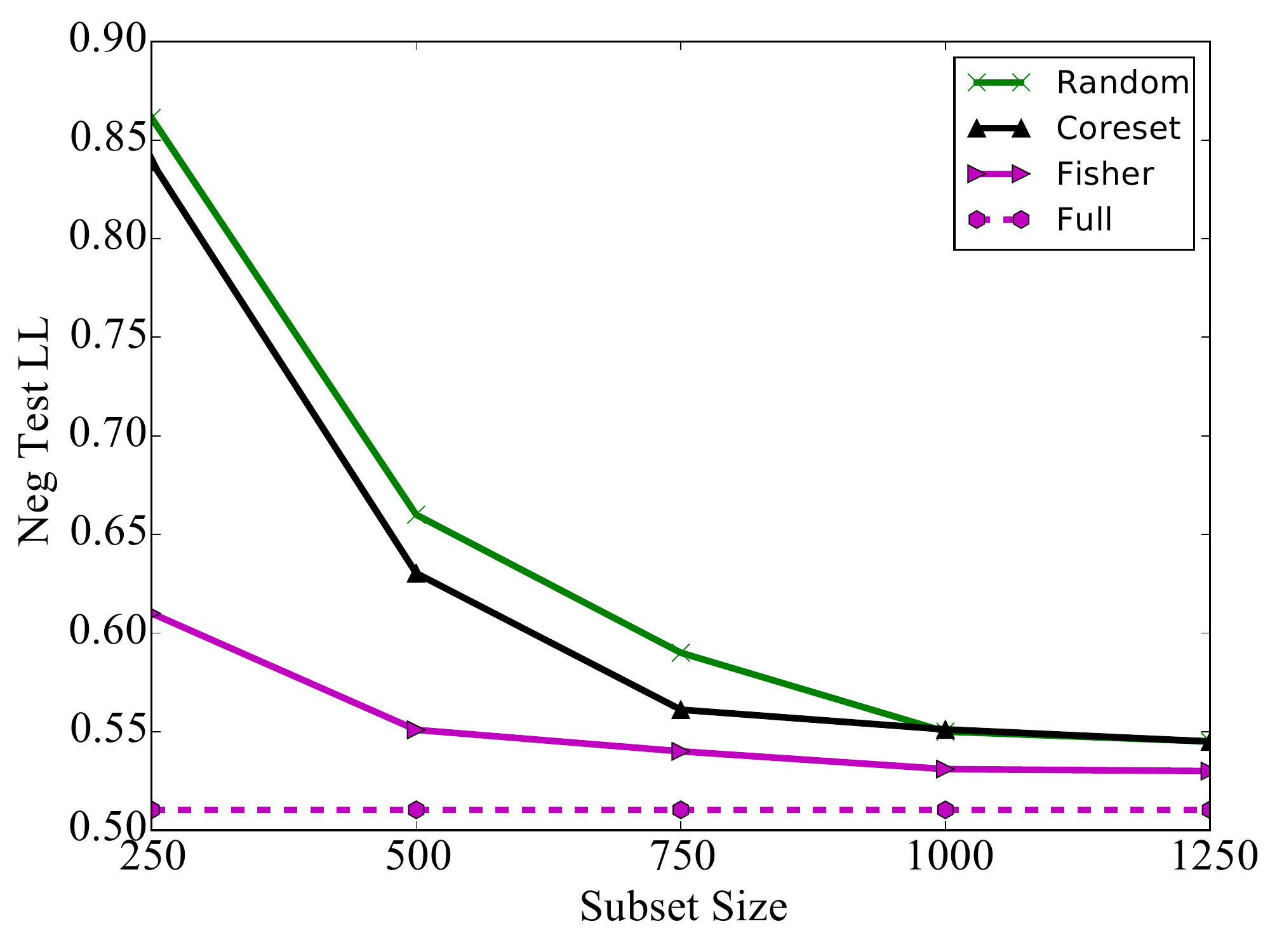}
 \end{subfigure}
 \caption{Performance for logistic regression over two datasets (left is \texttt{ChemReact} while right is \texttt{CovType}) of our method (Fisher) vs coreset selection~\citep{HugginsCB16} and random data selection. `Full' reports the numbers for training with the entire training set. Fisher achieves much better test LL performance than the baselines over several different subset sizes.}
 \label{fig:coresets}
 %\vspace{-4mm}
\end{figure}
% \vspace{-3mm}

%% file: conclusion.tex
%\section{Conclusion}\label{sec:conclusion}
\paragraph{Conclusion:} 
This manuscript proposed a novel principled approach for examining sets of training examples that influence an entire test set given a trained black-box model -- extending a notable recently proposed per-example influence to set-wise influence. We also presented novel convergence guarantees for SBQ and more scalable algorithm variants.. Empirical results were presented to highlight the utility of the proposed approach for black-box model interpretability and related tasks.
For future work, 
%we plan to further improve the computational efficiency of example selection and the efficiency of the fisher kernel evaluation. We also 
we plan to investigate the use of model criticisms to provide additional insights into the trained models.

%% file: appendix.tex
\section{Appendix}

\subsection{Proof of Theorem~\ref{thm:convergence}}.
Our proof follows the following sketch. We show that the given problem can be written as a linear regression problem in the induced RKHS. The greedy SBQ algorithm to choose data points is then equivalent to forward greedy feature selection in the transformed space (Lemma~\ref{lem:mmd_bq}). After the selection is made, the weight optimization obtained through the posterior calculation ensures orthogonal projection (Lemma~\ref{lem:orthoproj}) which means the posterior calculation is nothing but fitting of the least squares regression on the chosen set of features. Finally we draw upon research in discrete optimization to get approximation guarantees for greedy feature selection for least squares regression (Lemma~\ref{lem:greedy_weaksubmod_guarantee}) that we use to obtain the convergence rates.

We will require the following definition of the Maximum Mean Discrepancy (MMD). MMD is a divergence measure between two distributions $p$ and $q$ over a class of functions $\cF$. We restrict our attention to cases when $\cF$ is a Reproducing Kernel Hilbert Space (RKHS), which allows MMD evaluation based only on kernels, rather than explicit function evaluations. 

\[
\text{MMD}_\cF(p,q) := \sup_{f \in \cF} \left| \int f(\bx) p(\bx)d\bx - \int f(\bx) q(\bx)d\bx \right| 
\]

If the sup is reached at $\phi$, 
\begin{align*}
\text{MMD}_\cF(p,q)^2 &= \left| \int \phi(\bx) p(\bx)d\bx - \int \phi(\bx) q(\bx)d\bx\right|^2 \\
&= \| \mu_p -\mu_q \|_\cF^2 ,
\end{align*}

where $\mu_p$ and $\mu_q$ are the mean function mappings under $p$ and $q$ respectively. 

We make use of the following lemma that establishes a connection between MMD and Bayesian Quadrature. 

\begin{lemma}~\citep{Huszar2012OptimallyWeightedHI} \label{lem:mmd_bq} Let $q$ be the distribution established by weights $w_i$ of the Bayesian Quadrature over the selected points. Then, the expected variance of the weighted sum in Bayesian Quadration~\eqref{eq:sbq_costfunction} is equal to $\text{MMD}^2(p,q)$. 
\end{lemma}

We can make this explicit in our notation. If $\cF$ is an RKHS, we can write the MMD cost function using only the kernel function $K(\cdot, \cdot)$ associated with the RKHS~\cite{gretton2008kernel} as:-

\begin{align*}
\text{MMD}_\cF(p,q)^2 &=  \int_{\bx \sim p}\int_{\by \sim p} K(\bx, \by) p(\bx)p( \by)d\bx d\by  \\ 
&- 2 *\int_{\bx \sim p} \int_{\by \sim q} K(\bx,\by) p(\bx) q(\by)d\bx \by  \\
&+  \int_{\bx \sim q} \int_{\by \sim q} K(\bx,\by) q(\bx) q(\by)d\bx \by \\
&= \| \mu_p -\sum_i w_i \phi(\bx_i) \|_\cF^2 ,
\end{align*}

where, $\phi(\cdot)$ represents the feature mapping under the kernel function $K(\cdot, \cdot)$, and $i$ ranges over the selected points that define our discrete distribution $q$. Recall that Bayesian Quadrature deviates from simple kernel herding by allowing for and optimizing over non-uniform weights $w_i$. We can formally show that the weight optimization obtained through the posterior calculation performs an orthogonal projection of $\mu_p$ onto the span of selected points to get $\mu_q$ in the induced kernel space. 

\begin{lemma}\label{lem:orthoproj} The weights obtained  $w_i$ through the posterior evaluation of $Z(\sfS_n)$ guarantee that $ \sum_i w_i \phi(\bx_i)$ is the orthogonal projection of $\mu_p$ onto span($\phi(\bx_i)$.
\end{lemma}
\begin{proof}
Note that it suffices to show that the residual of the projection $\mu_p - \sum_j w_j \phi(\bx_j)$ is orthogonal to $\phi(\bx_i)$ for all $i$ in $\cH$. Recall that $w_i = \sum_j [\bK^{-1}]_{ij} \bz_j $, and $\bz_i = \int k(\bx,\bx_i) p(\bx) d(\bx)$. For an arbitrary index $i$, 

\begin{align*}
& \langle \mu_p - \sum_j w_j \phi(\bx_j), \phi(\bx_i) \rangle_\cH \\ 
= & \int k(\bx, \bx_i) p(\bx) d(\bx) - \langle \sum_j w_j \phi(\bx_j), \phi(\bx_i)\rangle_\cH \\ 
= & \bz_i  - \langle \sum_j w_j \phi(\bx_j), \phi(\bx_i)\rangle_\cH \\ 
= & \bz_i  - \sum_j w_j k(\bx_j, \bx_i)\\ 
= & \bz_i  - \sum_j \sum_t  [\bK^{-1}]_{tj} \bz_t  k(\bx_j, \bx_i)\\ 
= & \bz_i  - \sum_t \bz_t \sum_j  \bK_{ji} [\bK^{-1}]_{tj}   \\ 
= & \bz_i  - \bz_i ,  \\ 
\end{align*}  
where the last equality follows by noting that  $\sum_j \bK_{ji}  [\bK^{-1}]_{tj} $ is inner product of row $i$ of $\bK$ and row $t$ of $\bK^{-1}$ which is $1$ if $t=i$ and $0$ otherwise. This completes the proof. 
\end{proof}

Lemma~\ref{lem:orthoproj} implies that given the selected points, the posterior evaluation is equivalent to the optimizing for $\bw$ to minimize MMD$(p,q)^2$. In other words, the weight optimization is a simple linear regression in the mapped space $\cF$, and SBQ is equivalent to a greedy forward selection algorithm in $\cF$.

We shall also make use of recent results in generalization of submodular functions. Let $\frp(\sfS)$ be the power set of the set $\sfS$.

\begin{definition}[$\lambda$-weak submodular functions~\citep{Kempe2011,Elenberg2016}]\label{def:weaksubmodfn} A set function $g:\frp([n])\rightarrow \bbR$ is $\lambda$-weak submodular if $\exists \lambda > 0$ s.t. $\forall \sfL, \sfS \subset [n]$ $\sfL \cap \sfS = \emptyset$, 
\[
\sum_{j\in \sfS} \left[g(\sfL \cup \{j\}) - g(\sfL) \right] \geq \lambda \left[ g(\sfL \cup \sfS) - g(\sfL)\right]
\]
\end{definition}

Weak submodularity generalizes submodularity so that a greedy forward selection algorithm guarantees a $(1 - \nicefrac{1}{e^\lambda}) $ approximation for $\lambda$-weak submodular functions~\citep{Elenberg2016}. Standard submodular functions have a guarantee of $(1- \nicefrac{1}{e})$~\citep{nemhauser1978}. Thus, submodular functions are $1$-weak submodular. To provide guarantees for Algorithm~\ref{algo:greedy}, we show that the normalized set optimization function is $\frac{m}{M}$-weak submodular, where $m,M$ depend on the spectrum of the kernel matrix. 

%Recall our set optimization function (from~\eqref{eq:setcostfunction}) is:
%
%\begin{equation}
%\label{eq:setoptimizationproblem}
% g(\sfS) := \max_{\substack {\sfS \subset [n]} \\ | \sfS| \leq r} \bz_\sfS [\bK_{\sfS\sfS}^{-1}] \bz_\sfS
%\end{equation}
%
%Note that $g(\emptyset) = 0$, so $g(\cdot)$ is normalized. We shall make use of sparse eigenvalues for a matrix. For a matrix $\bA$, the smallest (largest) $k$-sparse eigenvalues as min (max) of $\frac{\bx^\top \bA \bx}{\bx^\top \bx}$ such that $\| \bx\|_0 \leq k$, and $\bx \neq 0$. We present our approximation guarantee next. 
%
%\begin{theorem}
%\label{thm:approxguarantee}
%Let $m$ be the smallest $2k$ sparse eigenvalue and $M$ be the largest $k+1$-sparse eigenvalues of the kernel matrix  $\bK$ of the training set. If $\sfS_G$ of size $k$ is the set returned by Algorithm~\ref{algo:greedy} and $\sfS^\star$ of size $r$ is the optimal solution of~\eqref{eq:setoptimizationproblem}, then, 
%\[ g(\sfS_G)  \geq \left(1 - \exp(- \frac{mk}{Mr}) \right)  g(\sfS^\star).\]
%\end{theorem}

\begin{lemma} \label{lem:greedy_weaksubmod_guarantee}~\citep{Kempe2011} The linear regression function is $\frac{m}{M}$-weak submodular where $m$ is the smallest $2r$ sparse eigenvalue and $M$ is the largest $r+1$-sparse eigenvalues of the dot product matrix of the features.
\end{lemma}

We note that Lemma~\ref{lem:greedy_weaksubmod_guarantee} as proposed and proved by~\citet{Kempe2011} is for the euclidean space. However, their results directly translate to general RKHS as long as the RKHS is bounded, or the candidate atoms have bounded norm. Hence under additional assumptions of bounded norm, the proofs and results of~\citet{Kempe2011} directly translate to general RKHS. 
From Lemma~\ref{lem:greedy_weaksubmod_guarantee} and recent results on weakly submodular functions, (\citep[Corollary 1]{Elenberg2016}), we get the following approximation guarantee for $g(\cdot)$ under the assumptions of Lemma~\ref{lem:greedy_weaksubmod_guarantee}.

\[ g(\sfS_G)  \geq \left(1 - \exp(- \frac{mk}{Mr}) \right)  g(\sfS^\star).\]

Setting $g(\sfS) = v(\phi) - v(\sfS) $, we get the final result.